\documentclass[10pt,twocolumn,letterpaper]{article}

\usepackage{cvpr}
\usepackage{times}
\usepackage{epsfig}
\usepackage{graphicx}
\usepackage{amsmath}
\usepackage{amssymb}
\usepackage{mathtools}

\usepackage{color}
\usepackage{subfigure}
\usepackage{booktabs}
\usepackage{amsthm}

\usepackage[breaklinks=true,bookmarks=false,hidelinks=true]{hyperref}
\usepackage[hyphenbreaks]{breakurl}

\usepackage{array}
\newcolumntype{L}[1]{>{\raggedright\let\newline\\\arraybackslash\hspace{0pt}}m{#1}}
\newcolumntype{C}[1]{>{\centering\let\newline\\\arraybackslash\hspace{0pt}}m{#1}}
\newcolumntype{R}[1]{>{\raggedleft\let\newline\\\arraybackslash\hspace{0pt}}m{#1}}

\newcommand{\specialcellbold}[2][c]{%
  \bfseries
  \begin{tabular}[#1]{@{}l@{}}#2\end{tabular}%
}

\cvprfinalcopy 

\ifcvprfinal\fi

\newtheorem*{claim}{Claim}

\begin{document}


\title{Locating Objects Without Bounding Boxes}

\author{Javier Ribera, David G\"{u}era, Yuhao Chen, Edward J. Delp\\
    {Video and Image Processing Laboratory (VIPER), Purdue University} \\
}

\maketitle
\ifcvprfinal\thispagestyle{empty}\fi

\begin{abstract}
    Recent advances in convolutional neural networks (CNN) have achieved remarkable results in locating objects in images.
    In these networks, the training procedure usually requires providing bounding boxes or the maximum number of expected objects.
    In this paper, we address the task of estimating object locations without annotated bounding boxes which are typically hand-drawn and time consuming to label.
    We propose a loss function that can be used in any fully convolutional network (FCN) to estimate object locations.
    This loss function is a modification of the average Hausdorff distance between two unordered sets of points.
    The proposed method has no notion of bounding boxes, region proposals, or sliding windows.
    We evaluate our method with three datasets designed to locate people's heads, pupil centers and plant centers.
    We outperform state-of-the-art generic object detectors and methods fine-tuned for pupil tracking.
\end{abstract}

\section{Introduction}

Locating objects in images is an important task in computer vision.
A common approach in object detection is to obtain bounding boxes around the objects of interest.
In this paper, we are not interested in obtaining bounding boxes.
Instead, we define the object localization task as obtaining a single 2D coordinate corresponding to the location of each object.
The location of an object can be any key point we are interested in, such as its center.
Figure~\ref{fig:collage} shows an example of localized objects in images.
Differently from other keypoint detection problems, we do not know in advance the number of keypoints in the image.
To also make the method as generic as possible we do not assume any physical constraint between the points, unlike in cases such as pose estimation.
This definition of object localization is more appropriate for applications where objects are very small, or substantially overlap (see the overlapping plants in Figure~\ref{fig:collage}).
In these cases, bounding boxes may not be provided by the dataset or they may be infeasible to groundtruth.

Bounding-box annotation is tedious, time-consuming and expensive~\cite{Papadopoulos2015}.
For example, annotating ImageNet~\cite{imagenet} required 42 seconds per bounding box when crowdsourcing on Amazon's Mechanical Turk using a technique specifically developed for efficient bounding box annotation~\cite{su2012}.
In \cite{bell2015}, Bell \etal~introduce a new dataset for material recognition and segmentation.
By collecting click location labels in this dataset instead of a full per-pixel segmentation, they reduce the annotation costs an order of magnitude.

\begin{figure}[t]
\begin{center}
   \includegraphics[width=0.5\textwidth]{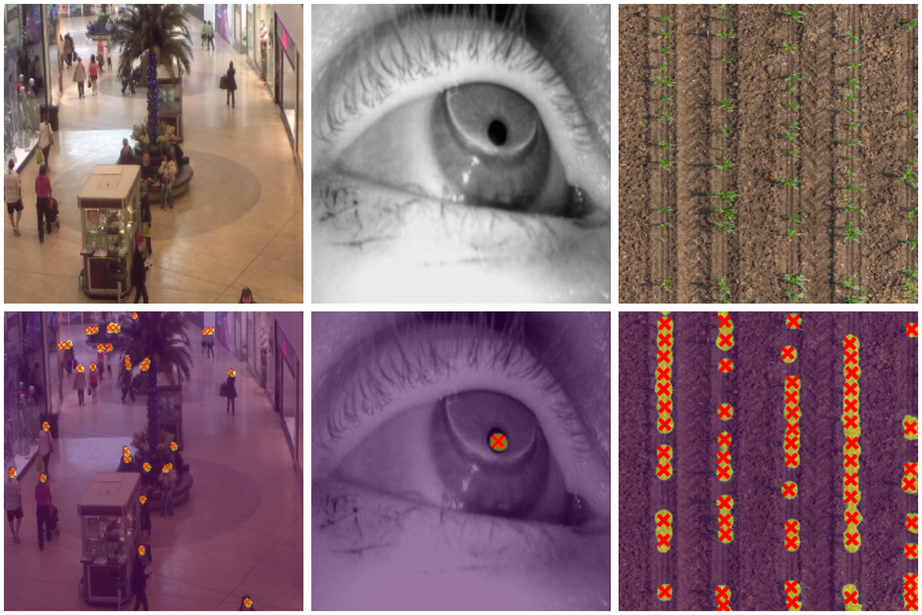}
\end{center}
    \caption{Object localization with human heads, eye pupils and plant centers.
             (Bottom) Heat map and estimations as crosses.}
\label{fig:collage}
\end{figure}

In this paper, we propose a modification of the average Hausdorff distance as a loss function of a CNN to estimate the location of objects.
Our method does not require the use of bounding boxes in the training stage, and does not require to know the maximum number of objects when designing the network architecture.
For simplicity, we describe our method only for a single class of objects, although it can trivially be extended to multiple object classes.
Our method is object-agnostic, thus the discussion in this paper does not include any information about the object characteristics.
Our approach maps input images to a set of coordinates, and we validate it with diverse types of objects.
We evaluate our method with three datasets.
One dataset contains images acquired from a surveillance camera in a shopping mall, and we locate the heads of people.
The second dataset contains images of human eyes, and we locate the center of the pupil.
The third dataset contains aerial images of a crop field taken from an Unmanned Aerial Vehicle (UAV), and we locate the centers of highly occluded plants.

Our approach to object localization via keypoint detection is not a universal drop-in replacement for bounding box detection,
specially for those tasks that inherently require bounding boxes, such as automated cropping.
Also, a limitation of this approach is that bounding box labeling incorporates some sense of scale, while keypoints do not.

The contributions of our work are:
\begin{itemize}
    \item We propose a loss function for object localization, which we name \textit{weighted Hausdorff distance} (WHD), that overcomes the limitations of pixelwise losses such as $L^2$ and the Hausdorff distances.
    \item We develop a method to estimate the location and number of objects in an image, without any notion of bounding boxes or region proposals.
    \item We formulate the object localization problem as the minimization of distances between points, independently of the model used in the estimation.
          This allows to use any fully convolutional network architectural design.
    \item We outperform state-of-the-art generic object detectors and achieve comparable results with crowd counting methods without any domain-specific knowledge, data augmentation, or transfer learning.
\end{itemize}

\section{Related Work}
\textbf{Generic object detectors.}
Recent advances in deep learning~\cite{deeplearningbook,nature} have increased the accuracy of localization tasks such as object or keypoint detection.
By generic object detectors, we mean methods that can be trained to detect any object type or types, such as Faster-RCNN~\cite{fastrcnn}, Single Shot MultiBox Detector (SSD)~\cite{ssd}, or YOLO~\cite{redmon2016you}.
In Fast R-CNN, candidate regions or proposals are generated by classical methods such as selective search~\cite{selectivesearch}.
Although activations of the network are shared between region proposals, the system cannot be trained end-to-end.
Region Proposal Networks (RPNs) in object detectors such as Faster R-CNN~\cite{fastrcnn,fasterrcnn} allow for end-to-end training of models.
Mask R-CNN~\cite{maskrcnn} extends Faster R-CNN by adding a branch for predicting an object mask but it runs in parallel with the existing branch for bounding box recognition.
Mask R-CNN can estimate human pose keypoints by generating a segmentation mask with a single class indicating the presence of the keypoint.
The loss function in Mask R-CNN is used location by location, making the keypoint detection highly sensitive to alignment of the segmentation mask.
SDD provides fixed-sized bounding boxes and scores indicating the presence of an object in the boxes.
The described methods either require groundtruthed bounding boxes to train the CNNs or require to set the maximum number of objects in the image being analyzed.
In \cite{huang2017}, it is observed that generic object detectors such as Faster R-CNN and SSD perform very poorly for small objects.

\textbf{Counting and locating objects.}
Counting the number of objects in an image is not a trivial task.
In~\cite{lempitsky_2010}, Lempitsky \etal~estimate a density function whose integral corresponds to the object count.
In~\cite{shao_2015}, Shao \etal~proposed two methods for locating objects.
One method first counts and then locates, and the other first locates and then counts.

Locating and counting people is necessary for many applications such as crowd monitoring in surveillance systems, surveys for new businesses, and emergency management \cite{lempitsky_2010,xiong2017}.
There are multiple studies in the literature, where people in videos of crowds are detected and tracked~\cite{track_by_detection2008,breitenstein2011}.
These detection methods often use bounding boxes around each human as ground truth.
Acquiring bounding boxes for each person in a crowd can be labor intensive and imprecise under conditions where lots of people overlap, such as sports events or rush-hour agglomerations in public transport stations.
More modern approaches avoid the need of bounding boxes by estimating a density map whose integral yields the total crowd count.
In approaches that involve a density map, the label of the density map is constructed from the labels of the people's heads.
This is typically done by centering Gaussian kernels at the location of each head.
Zhang \etal~\cite{zhang2016} estimate the density image using a multi-column CNN that learns features at different scales.
In \cite{sam2017}, Sam \etal~use multiple independent CNNs to predict the density map at different crowd densities.
An additional CNN classifies the density of the crowd scene and relays the input image to the appropriate CNN.
Huang \etal~\cite{huang2018} propose to incorporate information about the body part structure to the conventional density map to reformulate the crowd counting as a multi-task problem.
Other works such as Zhang \etal~\cite{zhang2015} use additional information such as the groundtruthed perspective map.

Methods for pupil tracking and precision agriculture are usually domain-specific.
In pupil tracking, the center of the pupil must be resolved in images obtained in real-world illumination conditions~\cite{fuhl2015}.
A wide range of applications, from commercial applications such as video games~\cite{eyetracking_book}, driving~\cite{applegan,copilot} or microsurgery~\cite{fuhl2016_2} rely on accurate pupil tracking.
In remote precision agriculture, it is critical to locate the center of plants in a crop field.
Agronomists use plant traits such as plant spacing to predict future crop yield~\cite{thornley_1983,sui_2011,tokatlidis_2004,farnham2001,chauhan2011}, and plant scientists to breed new plant varieties~\cite{araus_2014,neilson_2015}.
In~\cite{aich2018}, Aich \etal count wheat plants by first segmenting plant regions and then counting the number of plants in each segmented patch.

\textbf{Hausdorff distance.}
The Hausdorff distance can be used to measure the distance between two sets of points~\cite{attouch1991}.
Modifications of the Hausdorff distance~\cite{dubuisson1994} have been used for various multiple tasks, including character recognition~\cite{lu2001}, face recognition~\cite{lin2003} and scene matching~\cite{lin2003}.
Schutze \etal~\cite{schutze2012} use the average Hausdorff distance to evaluate solutions in multi-objective optimization problems.
In~\cite{elkhiyari2017}, Elkhiyari \etal~compare features extracted by a CNN according to multiple variants of the Hausdorff distance for the task of face recognition.
In~\cite{fan2017}, Fan \etal~use the Chamfer and Earth Mover's distance, along with a new neural network architecture, for 3D object reconstruction by estimating the location of a fixed number of points.
The Hausdorff distance is also a common metric to evaluate the quality of segmentation boundaries in the medical imaging community~\cite{taha2015,zhou2017,liao2013,teikari2016}.

\section{The Average Hausdorff Distance}
Our work is based on the Hausdorff distance which we briefly review in this section.
Consider two unordered non-empty sets of points $X$ and $Y$ and a distance metric $d(x, y)$ between two points $x \in X$ and $y \in Y$.
The function $d(\cdot, \cdot)$ could be any metric.
In our case we use the Euclidean distance.
The sets $X$ and $Y$ may have different number of points.
Let $\Omega \subset \mathbb{R}^2$ be the space of all possible points.
In its general form, the Hausdorff distance between $X \subset \Omega$ and $Y \subset \Omega$ is defined as
\begin{equation}
  \label{eq:hausdorff}
  d_{\text{H}}(X, Y) = \max \left\{ \sup_{x\in X} \inf_{y\in Y} d(x, y), \sup_{y\in Y} \inf_{x\in X} d(x, y) \right\}  .
\end{equation}

When considering a discretized and bounded $\Omega$, such as all the possible pixel coordinates in an image, the suprema and infima are achievable and become maxima and minima, respectively.
This bounds the Hausdorff distance as
\begin{equation}
  \label{eq:dmax}
  d(X, Y) \le d_{max} = \max_{x\in \Omega, y \in \Omega} d(x, y)  ,
\end{equation}
which corresponds to the diagonal of the image when using the Euclidean distance.
As shown in~\cite{attouch1991}, the Hausdorff distance is a metric.
Thus $\forall X, Y, Z \subset \Omega$ we have the following properties:
\begin{subequations}
\label{all1}
 \begin{align}
  \label{eq:hausdorff1}
     d_H(X, Y) & \ge 0 \quad   \\
  \label{eq:hausdorff2}
  d_H(X, Y) &= 0 \iff X = Y  \\
  \label{eq:hausdorff3}
  d_H(X, Y) &= d_H(Y, X)  \\
  \label{eq:hausdorff4}
  d_H(X, Y) &\le d_H(X, Z) + d_H(Z, Y)
 \end{align}
\end{subequations}

Equation~\eqref{eq:hausdorff2} follows from $X$ and $Y$ being closed, because in our task the pixel coordinate space $\Omega$ is discretized.
These properties are very desirable when designing a function to measure how similar $X$ and $Y$ are~\cite{arkin1991}.

\begin{figure}[t]
  \centering
  \includegraphics[width=0.3\textwidth]{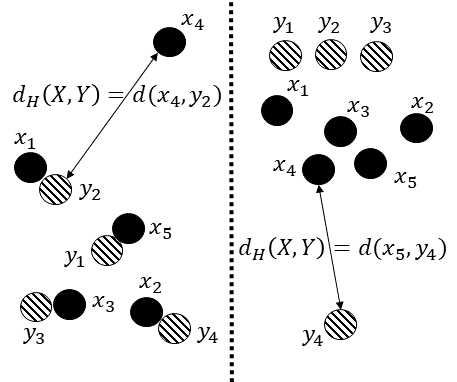}%

  \caption{Illustration of two different configurations of point sets $X=\{x_1, ..., x_5\}$ (solid dots) and $Y=\{y_1, ..., y_4\}$ (dashed dots).
           Despite the clear difference in the distances between points, their Hausdorff distance are equal because the worst outlier is the same.}
  \label{fig:points}
\end{figure}

A shortcoming of the Hausdorff function is its high sensitivity to outliers~\cite{schutze2012,taha2015}.
Figure~\ref{fig:points} shows an example for two finite sets of points with one outlier.
To avoid this, the average Hausdorff distance is more commonly used:
\begin{equation}
  \label{eq:AH}
  d_{\text{AH}}(X, Y) = \frac{1}{|X|} \sum_{x\in X} \min_{y\in Y} d(x, y) + \frac{1}{|Y|} \sum_{y\in Y} \min_{x\in X} d(x, y) ,
\end{equation}
where $|X|$ and $|Y|$ are the number of points in $X$ and $Y$, respectively.
Note that properties~\eqref{eq:hausdorff1},~\eqref{eq:hausdorff2} and ~\eqref{eq:hausdorff3} are still true, but~\eqref{eq:hausdorff4} is not.
Also, the average Hausdorff distance is differentiable with respect to any point in $X$ or $Y$.

Let $Y$ contain the ground truth pixel coordinates, and $X$ be our estimation.
Ideally, we would like to use $d_{\text{AH}}(X, Y)$ as the loss function during the training of our convolutional neural network (CNN).
We find two limitations when incorporating the average Hausdorff distance as a loss function.
First, CNNs with linear layers implicitly determine the estimated number of points $|X|$ as the size of the last layer.
This is a drawback because the actual number of points depends on the content of the image itself.
Second, FCNs such as U-Net~\cite{unet} can indicate the presence of an object center with a higher activation in the output layer, but they do not return the pixel coordinates.
In order to learn with backpropagation, the loss function must be differentiable with respect to the network output.

\section{The Weighted Hausdorff Distance}
\label{sec:whd}
To overcome these two limitations, we modify the average Hausdorff distance as follows:

\begin{equation}
  \label{eq:WH}
    \begin{split}
        d_{\text{WH}}(p, Y) = &\frac{1}{\mathcal{S}+ \epsilon} \sum_{x\in \Omega} p_x \min_{y\in Y} d(x, y) + \\
        & \frac{1}{|Y|} \sum_{y\in Y} \underset{x\in \Omega}{M_\alpha} \left[\> p_x d(x, y) + (1 - p_x) d_{max} \, \right],
    \end{split}
\end{equation}
where
\begin{equation}
  \label{eq:denom}
    \mathcal{S} = \sum_{x\in \Omega} p_x ,
\end{equation}

\begin{equation}
  \label{eq:genmean}
      \underset{a\in A}{M_\alpha} \left[ f(a) \right] = \left( \frac{1}{|A|} \sum_{a\in A} f^\alpha(a) \right) ^\frac{1}{\alpha} ,
\end{equation}

is the generalized mean, and $\epsilon$ is set to $10^{-6}$.
We call $d_{\text{WH}}(p, Y)$ the weighted Hausdorff distance (WHD).
$p_x \in [0, 1]$ is the single-valued output of the network at pixel coordinate $x$.
The last activation of the network can be bounded between zero and one by using a sigmoid non-linearity.
Note that $p$ does not need to be normalized, i.e., $\sum_{x\in \Omega} p_x = 1$ is not necessary.
Note that the generalized mean $M_\alpha \left[ \cdot \right]$ corresponds to the minimum function when $\alpha=-\infty$.
We justify the modifications applied to Equation~\eqref{eq:AH} to obtain Equation~\eqref{eq:WH} as follows:
\begin{enumerate}
  \item The $\epsilon$ in the denominator of the first term provides numerical stability when $p_x \approx 0 ~ \forall x \in \Omega$.
  \item When $p_x = \{0, 1\}$, $\alpha=-\infty$, and $\epsilon = 0$, the weighted Hausdorff distance becomes the average Hausdorff distance.
        We can interpret this as the network indicating with complete certainty where the object centers are.
        As $d_{\text{WH}}(p, Y) \geq 0$, the global minimum ($d_{\text{WH}}(p, Y) = 0$) corresponds to $p_x = 1$ if $x\in Y$ and 0 otherwise.
  \item In the first term, we multiply by $p_x$ to penalize high activations in areas of the image where there is no ground truth point $y$ nearby.
        In other words, the loss function penalizes estimated points that should not be there.
  \item In the second term, by using the expression \\ $f(\cdot) \vcentcolon = p_x d(x, y) + (1 - p_x) d_{max}$ we enforce that
      \begin{enumerate}
          \item If $p_{x_0} \approx 1$, then $f(\cdot) \approx d(x_0, y)$. This means the point $x_0$ will contribute to the loss as in the AHD (Equation~\eqref{eq:AH}).
          \item If $p_{x_0} \approx 0$, $x_0 \neq y$, then $f(\cdot) \approx d_{max}$.
              Then, if $\alpha = -\infty$, the point $x_0$ will not contribute to the loss because the ``minimum'' $M_{x\in \Omega}[\thinspace\cdot\thinspace]$ will ignore $x_0$.
              If another point $x_1$ closer to $y$ with $p_{x_1} > 0$ exists, $x_1$ will be ``selected'' instead by $M[\thinspace\cdot\thinspace]$.
              Otherwise $M_{x\in \Omega}[\thinspace\cdot\thinspace]$ will be high.
              This means that low activations around ground truth points will be penalized.
      \end{enumerate}
      Note that $f(\cdot)$ is not the only expression that would enforce these two constraints ($f|_{p_x = 1} = d(x, y)$ and $f|_{p_x = 0} = d_{max}$).
      We chose a linear function because of its simplicity and numerical stability.
\end{enumerate}

Both terms in the WHD are necessary.
If the first term is removed, then the trivial solution is $p_x = 1 \quad \forall x \in \Omega$.
If the second term is removed, then the trivial solution is $p_x = 0 \quad \forall x \in \Omega$.
These two cases hold for any value of $\alpha$ and the proof can be found in the appendix.
Ideally, the parameter $\alpha \to -\infty$ so that $M_\alpha(\cdot) = ||\cdot||_{-\infty}$ becomes the minimum operator \cite{minimum}.
However, this would make the second term flat with respect to the output of the network.
For a given $y$, changes in $p_{x_0}$ in a point $x_0$ that is far from $y$ would be ignored by $M_{-\infty}(\cdot)$, if there is another point $x_1$ with high activation and closer to $y$.
In practice, this makes training difficult because the minimum is not a smooth function with respect to its inputs.
Thus, we approximate the minimum with the generalized mean $M_\alpha(\cdot)$, with $\alpha<0$.
The more negative $\alpha$ is, the more similar to the AHD the WHD becomes, at the expense of becoming less smooth.
In our experiments, $\alpha=-1$.
There is no need to use $M_\alpha(\cdot)$ in the first term
because $p_x$ is not inside the minimum, thus the term is already differentiable with respect to $p$.

If the input image needs to be resized to be fed into the network, we can normalize the WHD to account for this distortion.
Denote the original image size as $(S_o^{(1)}, S_o^{(2)})$ and the resized image size as $(S_r^{(1)}, S_r^{(2)})$.
In Equation~\eqref{eq:WH}, we compute distances in the original pixel space by replacing $d(x, y)$ with $d(\textbf{S} x, \textbf{S} y)$, where $x, y \in \Omega$ and 
\begin{equation}
  \label{eq:normaliz}
    \textbf{S} = 
        \begin{pmatrix}
            S_o^{(1)} / S_r^{(1)}   &     0     \\
            0                             &     S_o^{(2)} / S_r^{(2)}  \\
        \end{pmatrix} .
\end{equation}

\subsection{Advantage Over Pixelwise Losses}
A naive alternative is to use a one-hot map as label, defined as $l_x = 1$ for $x \in Y$ and $l_x = 0$ otherwise,
and then use a pixelwise loss such as the Mean Squared Error (MSE) or the $L^2$ norm,
where $L^2(l, p) = \sum_{\forall x \in \Omega} |p_x - l_x|^2 \propto \text{MSE}(l, x)$.
The issue with pixelwise losses is that they are not informative of how close two points $x\in \Omega$ and $y\in Y$ are unless $x = y$.
In other words, it is flat for the vast majority of the pixels, making training unfeasible.
This issue is locally mitigated in \cite{tompson2015efficient} by using the MSE loss with Gaussians centered at each $x\in Y$.
By contrast, the WHD in Equation~\eqref{eq:WH} will decrease the closer $x$ is to $y$, making the loss function informative outside of the global minimum.

\section{CNN Architecture And Location Estimation}
\label{sec:cnn}

\begin{figure}[t]
\begin{center}
   \includegraphics[width=0.5\textwidth]{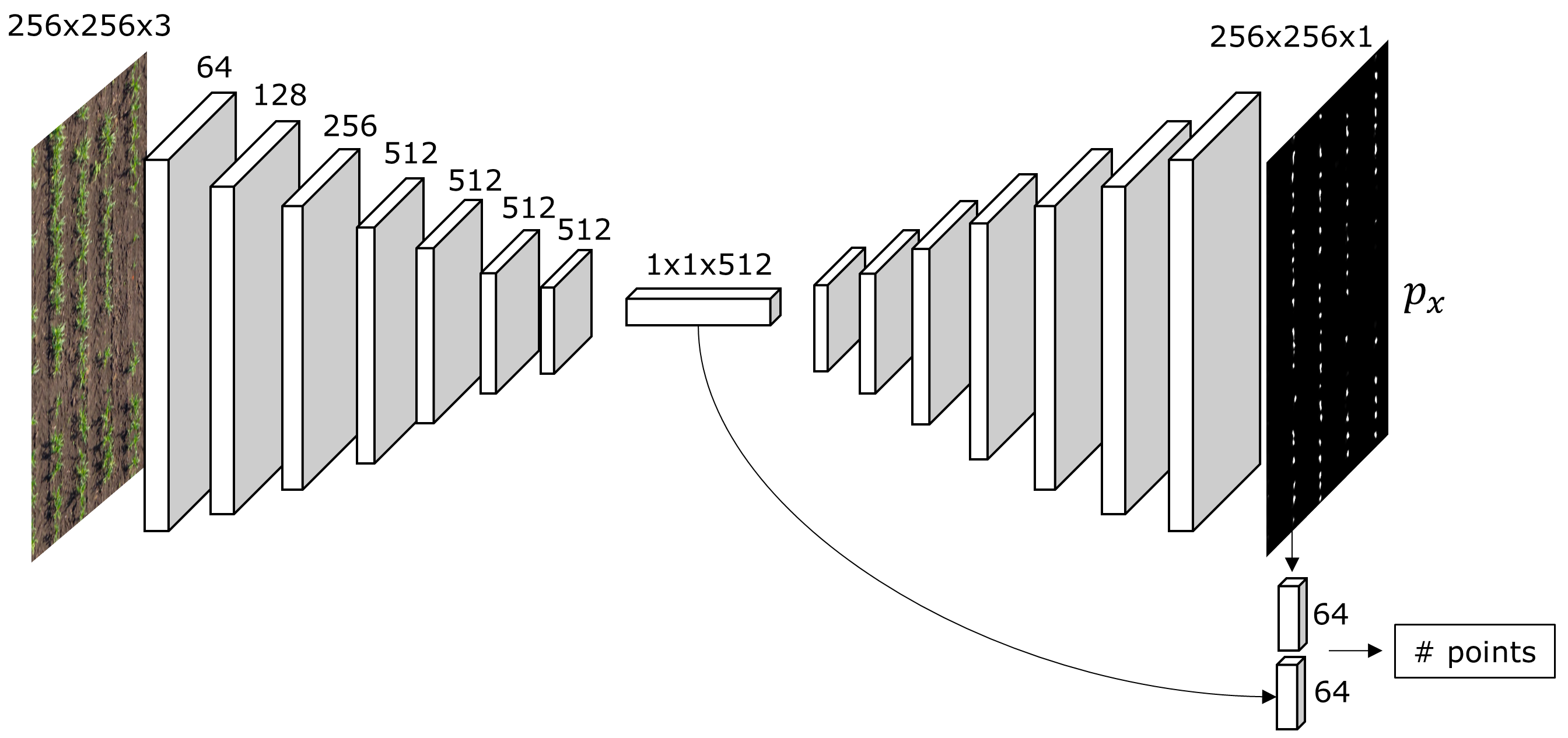}
\end{center}
   \caption{The FCN architecture used for object localization, minimally adapted from the U-Net~\cite{unet} architecture.
            We add a small fully-connected layer that combines the deepest features and the estimated probability map to regress the number of points.}
\label{fig:unet_like}
\end{figure}
In this section, we describe the architecture of the fully convolutional network (FCN) we use, and how we estimate the final object locations.
We want to emphasize that the network design is not a meaningful contribution of this work,
thus we have not made any attempt to optimize it.
Our main contribution is the use of the weighted Hausdorff distance as the loss function.
We adopt the U-Net architecture \cite{unet} and modify it minimally for this task.
Networks similar to U-Net have been proven to be capable of accurately mapping the input image into an output image, when trained in a conditional adversarial network setting~\cite{pix2pix} or when using a carefully tuned loss function~\cite{unet}.
Figure~\ref{fig:unet_like} shows the hourglass design of U-Net.
The residuals connections between each layer in the encoder and its symmetric layer in the decoder are not shown for simplicity.

This FCN has two well differentiated blocks.
The first block follows the typical architecture of a CNN.
It consists of the repeated application of two $3\times3$ convolutions (with padding 1), each followed by a batch normalization operation and a Rectified Linear Unit (ReLU).
After the ReLU, we apply a $2\times 2$ max pooling operation with stride 2 for downsampling.
At each downsampling step we double the number of feature channels, starting with 64 channels and using 512 channels for the last 5 layers.

The second block consists of repeated applications of the following elements: a bilinear upsampling, a concatenation with the feature map from the downsampling block, and two $3 \times 3$ convolutions, each followed by a batch normalization and a ReLU.
The final layer is a convolution layer that maps to the single-channel output of the network, $p$.

To estimate the number of objects in the image, we add a branch that combines the information from the deepest level features and also from the estimated probability map.
This branch combines both features (the $1\times 1 \times 512$ feature vector and the $256 \times 256$ probability map) into a hidden layer, and uses the 128-dimensional feature vector to output a single number.
We then apply a ReLU to ensure the output is positive, and round it to the closest integer to obtain our final estimate of the number of objects, $\hat{C}$.

Although we use this particular network architecture, any other architecture could be used.
The only requirement is that the output images of the network must be of the same size as the input image.
The choice of a FCN arises from the natural interpretation of its output as the weights ($p_x$) in the WHD (Equation~\eqref{eq:WH}).
In previous works \cite{elkhiyari2017,fan2017}, variants of the average Haussdorf distance were successfully used with non-FCN networks that estimate the point set directly.
However, in those cases the size of the estimated set is fixed by the size of the last layer.
To locate an unknown number of objects, the network must be able to estimate a variable number of object locations.
Thus, we could envision the WHD also being used in non-FCN networks as long as the output of the network is used as $p$ in Equation \eqref{eq:WH}.

The training loss we use to train the network is a combination of Equation \eqref{eq:WH} and a smooth $L_1$ loss for the regression of the object count.
The final training loss is
\begin{equation}
  \label{eq:loss}
\begin{split}
    \mathcal{L}(p, Y) = d_{WH}(p, Y) + \mathcal{L}_{\text{reg}}(C - \hat{C}(p)),
 \end{split}
\end{equation}

where $Y$ is the set containing the ground truth coordinates of the objects in the image, $p$ is the output of the network, $C = |Y|$, and $\hat{C}(p)$ is the estimated number of objects.
$\mathcal{L}_{\text{reg}}(\cdot)$ is the regression term, for which we use the smooth $L_1$ or Huber loss \cite{huber1964}, defined as
\begin{equation}
    \mathcal{L}_{\text{reg}}(x) =
    \begin{cases}
        0.5x^2,& \text{for} |x| < 1 \\
        |x| - 0.5,& \text{for} |x| \geq 1 \\
   \end{cases}
\end{equation}
This loss is robust to outliers when the regression error is high, and at the same time is differentiable at the origin.

\begin{figure}[t]
\begin{center}
   \includegraphics[width=0.5\textwidth]{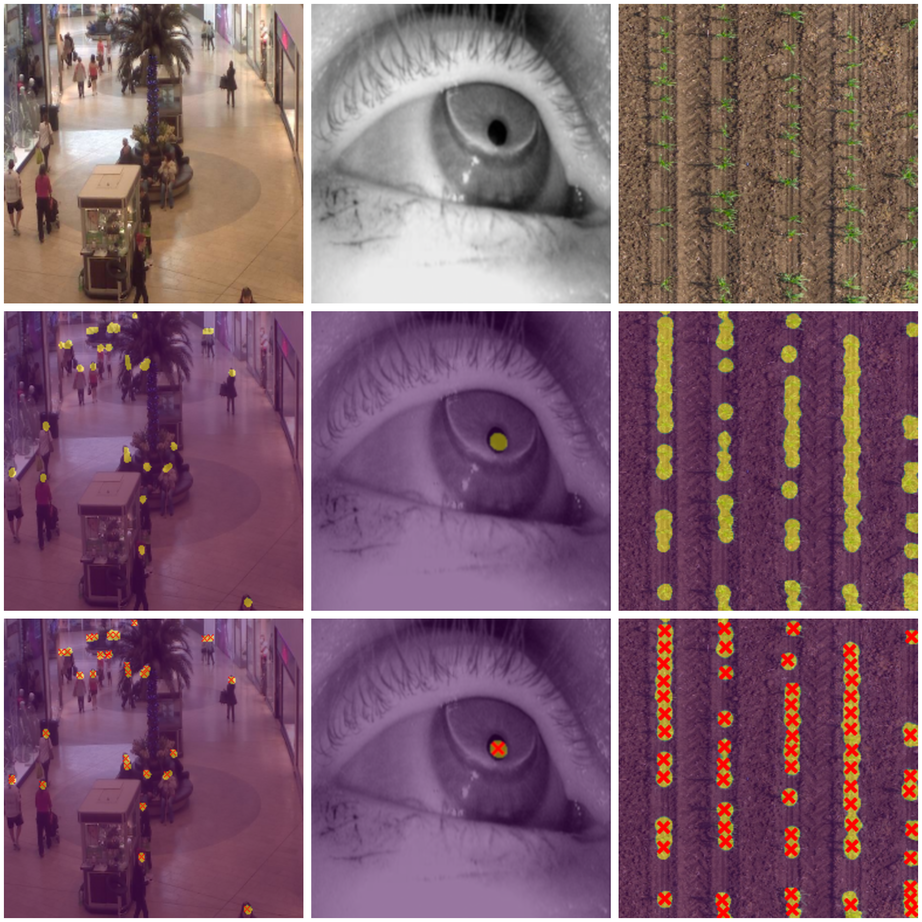}
\end{center}
   \caption{First row: Input image.
            Second row: Output of the network ($p$ in the text) overlaid onto the input image.
            This can be considered a saliency map of object locations.
            Third row: The estimated object locations are marked with a red cross.}
\label{fig:collage_clusters}
\end{figure}

The network outputs a saliency map $p$ indicating with $p_x \in [0,1]$ the confidence that there is an object at pixel $x$.
Figure \ref{fig:collage_clusters} shows $p$ in the second row.
During evaluation, our ultimate goal is to obtain $\hat{Y}$, i. e., the estimate of all object locations.
In order to convert $p$ to $\hat{Y}$, we threshold $p$ to obtain the pixels $T = \left\{ x\in \Omega ~ | ~ p_x > \tau \right\}$.
We can use three different methods to decide which $\tau$ to use:
\begin{enumerate}
    \item Use a constant $\tau$ for all images.
    \item Use Otsu thresholding \cite{otsu} to find an adaptive $\tau$ different for every image.
    \item Use a Beta mixture model-based thresholding (BMM).
          This method fits a mixture of two Beta distributions to the values of $p$ using the algorithm described in \cite{bmm}, and then takes the mean value of the distribution with highest mean as $\tau$.
\end{enumerate}
Figure \ref{fig:collage_clusters} shows in the third row an example of the result of thresholding the saliency map $p$.
Then, we fit a Gaussian mixture model to the points $T$.
This is done using the expectation maximization (EM) \cite{em} algorithm and the estimated number of plants $\hat{C}$.

The means of the fitted Gaussians are considered the final estimate $\hat{Y}$.
The third row of Figure \ref{fig:collage_clusters} shows the estimated object locations with red crosses.
Note that even if the map produced by the FCN is of good quality, i.e., there is a cluster on each object location, EM may not yield the correct object locations if $|\hat{C} - C | > 0.5$.
An example can be observed in the first column of Figure \ref{fig:collage_clusters}, where a single head is erroneously estimated as two heads.

\section{Experimental Results}
We evaluate our method with three datasets.

The first dataset consists of 2,000 images acquired from a surveillance camera in a shopping mall.
It contains annotated locations of the heads of the crowd.
This dataset is publicly available at \url{http://personal.ie.cuhk.edu.hk/~ccloy/downloads\_mall\_dataset.html} \cite{loy2013}.
80\%, 10\% and 10\% of the images were randomly assigned to the training, validation, and testing datasets, respectively.

The second dataset is presented in \cite{fuhl2015} with the roman letter V and publicly available at \url{http://www.ti.uni-tuebingen.de/Pupil-detection.1827.0.html}.
It contains 2,135 images with a single eye, and the goal is to detect the center of the pupil.
It was also randomly split into training, validation and testing datasets as 80/10/10 \%, respectively.

The third dataset consists of aerial images of a crop field taken from a UAV flying at an altitude of 40 m.
The images were stitched together to generate a $6,000 \times 12,000$ orthoimage of $0.75 $ cm/pixel resolution shown in Figure \ref{fig:ortho}.
The location of the center of all plants in this image was groundtruthed, resulting in a total of 15,208 unique plant centers.
This mosaic image was split, and the left 80\% area was used for training, the middle 10\% for validation, and the right 10\% for testing.
Within each region, random image crops were generated.
These random crops have a uniformly distributed height and width between 100 and 600 pixels.
We extracted 50,000 random image crops in the training region, $5,000$ in the validation region, and $5,000$ in the testing region.
Note that some of these crops may highly overlap.
We are making the third dataset publicly available at \url{https://engineering.purdue.edu/~sorghum/dataset-plant-centers-2016}.
We believe this dataset will be valuable for the community, as it poses a challenge due to the high occlusion between plants.

\begin{figure}[t]
\begin{center}
   \includegraphics[width=0.7\linewidth]{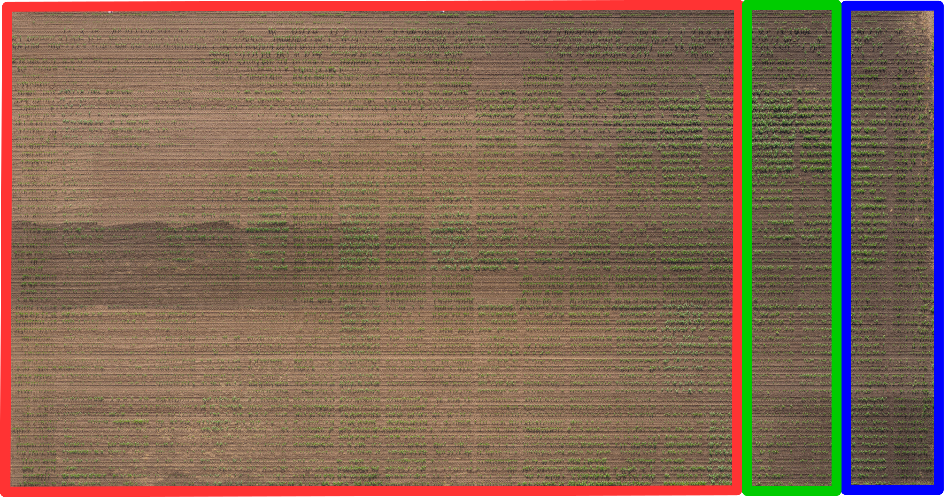}
\end{center}
   \caption{An orthorectified image of a crop field with 15,208 plants. The red region was used for training, the region in green for validation, and the region in blue for testing.}
\label{fig:ortho}
\end{figure}

All the images were resized to $256 \times 256$ because that is the minimum size our architecture allows.
The groundtruthed object locations were also scaled accordingly.
As for data augmentation, we only use random horizontal flip.
For the plant dataset, we also flipped the images vertically.
We set $\alpha=-1$ in Equation~\eqref{eq:genmean}.
We have also experimented with $\alpha=-2$ with no apparent improvement, but we did not attempt to find an optimal value.
We retrain the network for every dataset, i.e., we do not use pretrained weights.
For the mall and plant dataset, we used a batch size of 32 and Adam optimizer~\cite{kingma_2014, amsgrad} with a learning rate of $10^{-4}$ and momentum of 0.9.
For the pupil dataset, we reduced the size of the network by removing the five central layers, we used a batch size of 64, and stochastic gradient descent with a learning rate of $10^{-3}$ and momentum of 0.9.
At the end of each epoch, we evaluate the average Haussdorf distance (AHD) in Equation \eqref{eq:AH} over the validation set,
and select the epoch with lowest AHD on validation.

As metrics, we report Precision, Recall, F-score, AHD, Mean Absolute Error (MAE), Root Mean Squared Error (RMSE), and Mean Absolute Percent Error (MAPE): 
\begin{equation}
    \text{MAE} = \frac{1}{N}\sum_{i=1}^{N}| e_i |,  \quad \text{RMSE} = \sqrt{\frac{1}{N}\sum_{i=1}^{N} \big| e_i \big|^2}
  \label{eq:MAEandRMSE}
\end{equation}
\begin{equation}
    \text{MAPE} = 100 \frac{1}{N} \sum_{\substack{i=1 \\ C_i \neq 0}}^{N}\frac{\big| e_i \big|}{C_i}
  \label{eq:MAPE}
\end{equation}

where $e_i = \hat{C_i} - C_i$, $N$ is the number of images, $C_i$ is the true object count in the $i$-th image, and $\hat{C}_i$ is our estimate.

A true positive is counted if an estimated location is at most at distance $r$ from a ground truth point.
A false positive is counted if an estimated location does not have any ground truth point at a distance at most $r$.
A false negative is counted if a true location does have any estimated location at a distance at most $r$.
Precision is the proportion of our estimated points that are close enough to a true point.
Recall is the proportion of the true points that we are able to detect.
The F-score is the harmonic mean of precision and recall.
Note that one can achieve a precision and recall of 100\% even if we estimate more than one object location per ground truth point.
This would not be an ideal localization.
To take this into account, we also report metrics (MAE, RMSE and MAPE) that indicate if the number of objects is incorrect.
The AHD can be interpreted as the average location error in pixels.

Figure~\ref{fig:precnrec} shows the F-score as a function of $r$.
Note that $r$ is only an evaluation parameter.
It is not needed during training or testing.
MAE, RMSE, and MAPE are shown in Table \ref{tab:results1}.
Note that we are using the same architecture for all tasks, except for the pupil dataset, where we removed intermediate layers.
Also, in the case of the pupil detection, we know that there is always one object in the image.
Thus, regression is not necessary and we can remove the regression term in Equation~\eqref{eq:loss} and fix $\hat{C_i} = C_i = 1 ~ \forall i$.

A naive alternative approach to object localization would be to use generic object detectors such as Faster R-CNN \cite{fasterrcnn}.
One can train these detectors by constructing bounding boxes with fixed size centered at each labeled point.
Then the center of each bounding box can be taken as the estimated location.
We used bounding boxes of size $20 \times 20$ (the approximate average head and pupil size) and anchor sizes of $16\times 16$ and $32 \times 32$.
Note that these parameters may be suboptimal even though they were selected to match the type of object.
The threshold we used for the softmax scores was 0.5 and for the intersection over union it was 0.4,
because they minimize the AHD over the validation set.
We used the VGG-16 architecture~\cite{vgg} and trained it using stochastic gradient descent with learning rate of $10^{-3}$ and momentum of 0.9.
For the pupil dataset, we always selected the bounding box with the highest score.
We experimentally observed that Faster R-CNN struggles with detecting very small objects that are very close to each other.
Tables~\ref{tab:mall}-\ref{tab:plants} show the results of Faster R-CNN results on the mall, pupil, and plant datasets.
Note that the mall and plant datasets, with many small and highly overlapping objects, are the most challenging for Faster R-CNN.
This behaviour is consistent with the observations in \cite{huang2017},
where, all generic object detectors perform very poorly and Faster R-CNN yields a mean Average Precision (mAP) of 5\% in the best case.

We also experimented using mean shift \cite{meanshift} instead of Gaussian mixtures (GM) to detect the local maxima.
However, mean shift is prone to detect multiple local maxima, and GMs are more robust against outliers.
In our experiments, we observed that precision and recall were substantially worse than using GM.
More importantly, using Mean Shift slowed down validation an order of magnitude.
The average time for the Mean Shift algorithm to run on one of our images was 12 seconds, while fitting GM using expectation maximization took around 0.5 seconds, when using the scikit-learn implementations \cite{scikit-learn}.

\begin{figure}[t]
\centering     
\includegraphics[width=0.85\linewidth]{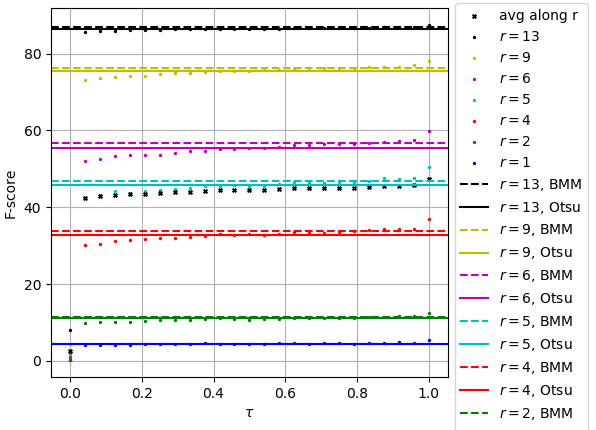}
    \caption{Effect on the F-score of the threshold $\tau$.}
    \label{fig:tau}
\end{figure}

\begin{figure}[t]
\centering     
\includegraphics[width=0.85\linewidth]{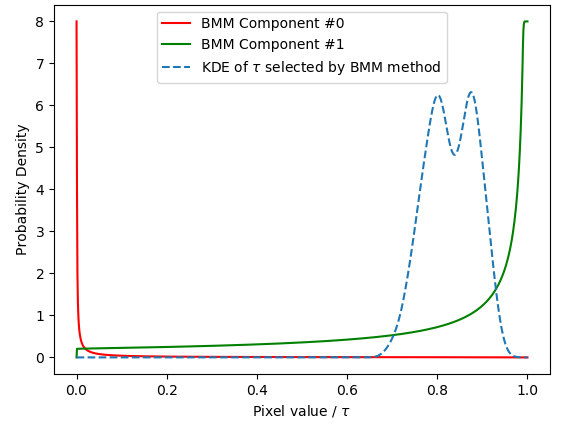}
    \caption{Beta mixture model fitted on the values of $p_x$, and the thresholds $\tau$ used by the BMM method.}
    \label{fig:bmm}
\end{figure}

We also investigated the effect of the parameter $\tau$, and the three methods to select it presented in Section~\ref{sec:cnn}.
One may think that this parameter could be a trade-off between some metrics, and that it should be cross-validated.
In practice, we observed that $\tau$ does not balance precision and recall, thus a precision-recall curve is not meaningful.
Instead, we plot the F-score as a function of $r$ in Figure~\ref{fig:precnrec}.
Also, cross-validating $\tau$ would imply fixing an ``optimal'' value for all images.
Figure~\ref{fig:tau} shows that we can do better with adaptive thresholding methods (Otsu or BMM).
Note that BMM thresholding (dashed lines) always outperforms Otsu (solid lines), and most of fixed $\tau$.
To justify the appropriateness of the BMM method, note that in Figure~\ref{fig:collage_clusters} most of the values in the estimated map are very high or very low.
This makes a Beta distribution a better fit than a Normal distribution (as used in Otsu's method) to model $p_x$.
Figure~\ref{fig:bmm} shows the fitted BMM and a kernel density estimation of the values of $\tau$ adaptively selected by the BMM method.

Lastly, as our method locates and counts objects simultaneously, it could be used as a counting technique.
We also evaluated our technique in the task of crowd counting using the ShanghaiTech Part B dataset presented in \cite{zhang2016},
and achieve a MAE of 19.9.
Even though we do not outperform state of the art methods that are specifically fine-tuned for crowd counting \cite{li2018},
we can achieve comparable results with our generic method.
We expect future improvements such as architectural changes or using transfer learning to further increase the performance.

A PyTorch implementation of the weighted Hausdorff distance loss and trained models are available at \url{https://github.com/javiribera/locating-objects-without-bboxes}.

\begin{figure}[t]
    \subfigure{\includegraphics[width=0.9\linewidth]{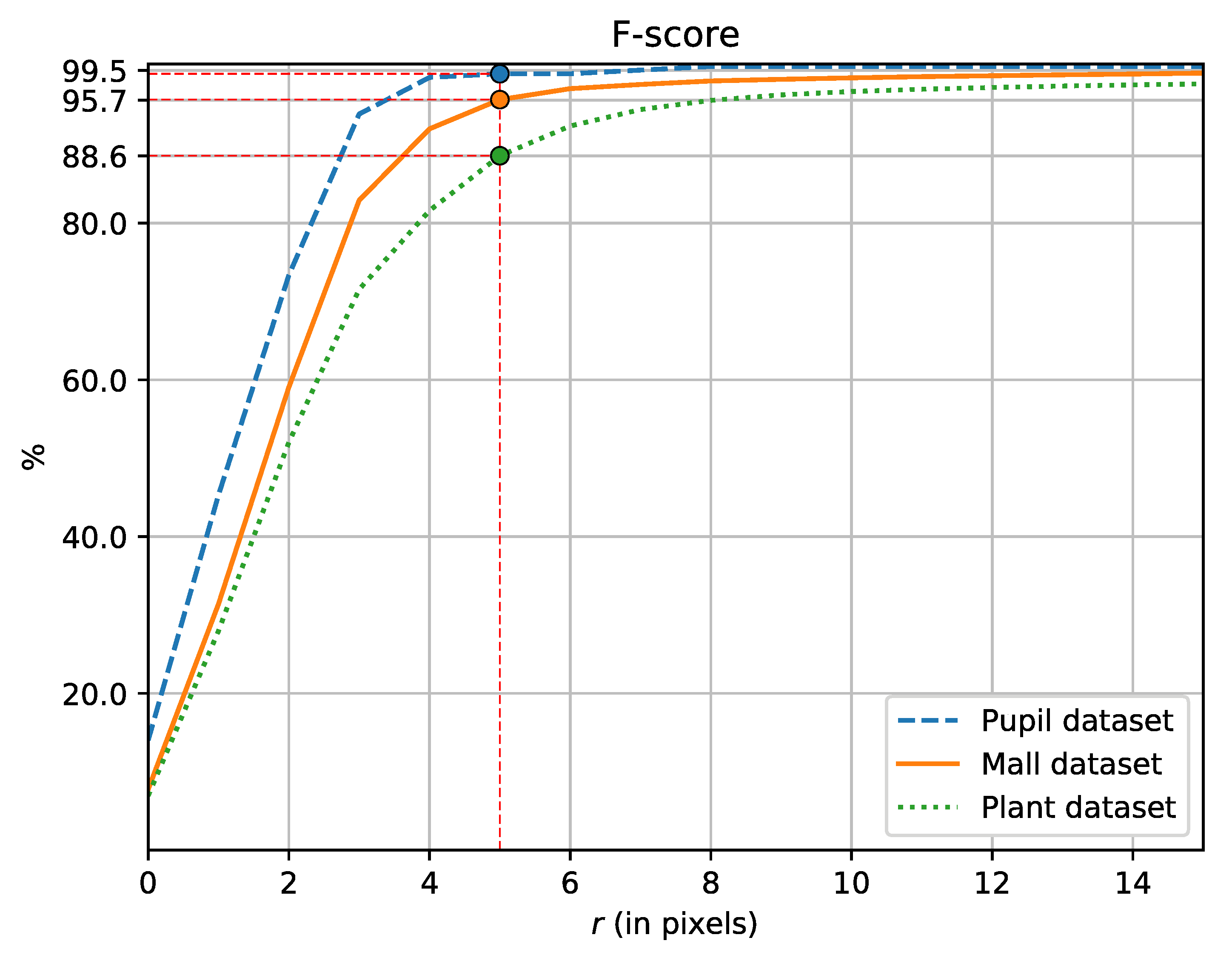}}
    \caption{F-score as a function of $r$,
             the maximum distance between a true and an estimated object location to consider it correct or incorrect.
             A higher $r$ makes correctly locating an object easier.}
    \label{fig:precnrec}
\end{figure}

\begin{table}[t]
\centering
\caption{Results of our method for object localization, using $r=5$.
         Metrics are defined in Equations~\eqref{eq:AH},~\eqref{eq:MAEandRMSE}-\eqref{eq:MAPE}.
         Regression metrics for the pupil dataset are not shown because there is always a single pupil ($\hat{C} = C = 1$).
         Figure \ref{fig:precnrec} shows the F-score for other $r$ values.}
\begin{tabular}{lllll}
\toprule
\textbf{Metric} & \specialcellbold{Mall \\ dataset} & \specialcellbold{Pupil \\ dataset}  & \specialcellbold{Plant \\ dataset} & \textbf{Average}\\
\midrule
Precision & 95.2\%  & 99.5\%  & 88.1\% & 94.4\%  \\
Recall & 96.2\%  & 99.5\% & 89.2\% & 95.0\%  \\
F-score & 95.7\%  & 99.5\% & 88.6\% & 94.6\%  \\
AHD & 4.5 px & 2.5 px   & 7.1 px & 4.7 px \\
MAE & 1.4 & -    & 1.9 & 1.7  \\
RMSE & 1.8 & -     & 2.7 & 2.3  \\
MAPE & 4.4\% & -    & 4.2\% & 4.3 \%  \\
\bottomrule
\end{tabular}
\label{tab:results1}
\end{table}

\begin{table}
\centering
\caption{Head location results using the mall dataset, using $r=5$.}
\begin{tabular}{lll}
\toprule
\textbf{Metric}  & \textbf{Faster-RCNN} &  \textbf{Ours} \\
\midrule
Precision & 81.1\%  & \textbf{95.2 \%} \\
Recall & 76.7\%   & \textbf{96.2 \%} \\
F-score & 78.8 \%   & \textbf{95.7 \%} \\
AHD & 7.6 px & \textbf{4.5 px} \\
MAE & 4.7 & \textbf{1.4}  \\
RMSE & 5.6 & \textbf{1.8} \\
MAPE & 14.8\% & \textbf{4.4 \%} \\
\bottomrule
\end{tabular}
\label{tab:mall}
\end{table}

\begin{table}\centering
\caption{Pupil detection results, using $r=5$.
         Precision and recall are equal because there is only one estimated and one true object.}
\begin{tabular}{llll}
\toprule
    \textbf{Method} & \textbf{Precision} & \textbf{Recall} &  \textbf{AHD} \\
    \midrule
    Swirski \cite{swirski2012} & 77 \% & 77 \%  & - \\
    ExCuSe \cite{fuhl2015} & 77 \% & 77 \%  & - \\
    Faster-RCNN & 99.5 \% & 99.5 \% &  2.7 px  \\
    \textbf{Ours} & \textbf{99.5} \% & \textbf{99.5} \% & \textbf{2.5 px}  \\
\bottomrule
\end{tabular}
\label{tab:pupil}
\end{table}

\begin{table}
\centering
\caption{Plant location results using the plant dataset, using $r=5$.}
\begin{tabular}{lll}
\toprule
\textbf{Metric}  & \textbf{Faster-RCNN} &  \textbf{Ours} \\
\midrule
Precision & 86.6 \%  & \textbf{88.1 \%} \\
Recall & 78.3 \%   & \textbf{89.2 \%} \\
F-score & 82.2 \%   & \textbf{88.6 \%} \\
AHD & 9.0 px & \textbf{7.1 px} \\
MAE & 9.4 & \textbf{1.9}  \\
RMSE & 13.4 & \textbf{2.7} \\
MAPE & 17.7 \% & \textbf{4.2 \%} \\
\bottomrule
\end{tabular}
\label{tab:plants}
\end{table}

\section{Conclusion}
We have presented a loss function for the task of locating objects in images that does not need bounding boxes.
This loss function is a modification of the average Hausdorff distance (AHD), which measures the similarity between two unordered sets of points.
To make the AHD differentiable with respect to the network output, we have considered the certainty of the network when estimating an object location.
The output of the network is a saliency map of object locations and the estimated number of objects.
Our method is not restricted to a maximum number of objects in the image, does not require bounding boxes, and does not use region proposals or sliding windows.
This approach can be used in tasks where bounding boxes are not available, or the small size of objects makes the labeling of bounding boxes impractical.
We have evaluated our approach with three different datasets, and outperform generic object detectors and task-specific techniques.
Future work will include developing a multi-class object location estimator in a single network,
and evaluating more modern CNN architectures.

\vspace{1em}

{\small \textbf{Acknowledgements:}
    This work was funded by the Advanced Research Projects Agency-Energy (ARPA-E),
    U.S. Department of Energy, under Award Number DE-AR0000593.
    The views and opinions of the authors
    expressed herein do not necessarily reflect those of the
    U.S. Government or any agency thereof.
    We thank Professor Ayman Habib for the orthophotos used in this paper.
    Contact information: Edward J. Delp, \texttt{\href{mailto:ace@ecn.purdue.edu}{ace@ecn.purdue.edu}}
}

\clearpage

\bibliographystyle{ieee}
\bibliography{ref}

\renewcommand{\thesection}{\Alph{section}}
\section*{Appendix: Ablation Of Terms}
In Section~\ref{sec:whd}, we made the following claim:
\begin{claim}
Both terms of the Weighted Hausdorff Distance (WHD) are necessary.
If the first term is removed, then
$p_x = 1 \enspace \forall x \in \Omega$ is the solution that minimizes the WHD.
If the second term is removed, then the trivial solution is $p_x = 0 \enspace \forall x \in \Omega$.
\end{claim}
\begin{proof}
    If the first term is removed and $p_x = 1 \enspace \forall x \in~\Omega$, then Equation~(5) reduces to
\begin{equation*}
    d_{\text{WH}}(p, Y) \big|_{p=1} = 
    \frac{1}{|Y|} \sum_{y\in Y} \underset{x\in \Omega}{M_\alpha} \left[\, d(x, y) \, \right] .
\end{equation*}
From the definition in Equation~(2), $\forall x,y \in \Omega$,
\begin{equation*}
    d(x, y) \leq  d_{max}.
\end{equation*}
For any $p_x \in [0, 1]$ and $\alpha < 0$,
\begin{equation*}
  \begin{split}
      (1-p_x) d(x, y) & \leq  (1-p_x) d_{max} \\
      d(x, y) & \leq  p_x d_{max} + (1-p_x) d_{max} \\
      d(x, y)^\alpha & \geq  \left[p_x d_{max} + (1-p_x) d_{max}\right]^\alpha \\
      \frac{1}{|\Omega|} \sum_{x \in \Omega} d(x, y)^\alpha & \geq  \frac{1}{|\Omega|} \sum_{x \in \Omega} \left[p_x d_{max} + (1-p_x) d_{max}\right]^\alpha \\
      \left[ \frac{1}{|\Omega|} \sum_{x \in \Omega} d(x, y)^\alpha \right]^\frac{1}{\alpha} & \leq  \left[ \frac{1}{|\Omega|} \sum_{x \in \Omega} \left[p_x d_{max} + (1-p_x) d_{max}\right]^\alpha \right]^\frac{1}{\alpha} \\
      \underset{x\in \Omega}{M_\alpha} \left[\, d(x, y) \, \right] & \leq  \underset{x\in \Omega}{M_\alpha} \left[\, p_x d_{max} + (1-p_x) d_{max} \, \right] \\
      \frac{1}{|Y|} \sum_{y \in Y} \underset{x\in \Omega}{M_\alpha} \left[\, d(x, y) \, \right] & \leq  \frac{1}{Y} \sum_{y \in Y} \underset{x\in \Omega}{M_\alpha} \left[\, p_x d_{max} + (1-p_x) d_{max} \, \right] \\
      d_{\text{WH}}(p, Y) \big|_{p=1} & \leq d_{\text{WH}}(p, Y) .
  \end{split}
\end{equation*}
Note that $d_{\text{WH}}(p, Y) \big|_{p=1} > 0$ if $\alpha > -\infty$, but the proof holds for any $\alpha < 0$.

If the second term is removed and $p_x = 0 \enspace \forall x \in \Omega$, then Equation~(5) reduces to
\begin{equation*}
    d_{\text{WH}}(p, Y) \big|_{p=0} = 
    \frac{1}{\mathcal{S}+ \epsilon} \sum_{x\in \Omega} p_x \min_{y\in Y} d(x, y) \big|_{p=0} 
    = \frac{1}{0 + \epsilon} 0 = 0 .
\end{equation*}
\end{proof}

\end{document}